\pdfoutput=1

\documentclass[11pt]{article}

\PassOptionsToPackage{dvipsnames,usename}{xcolor}
\usepackage{acl}

\usepackage{times}
\usepackage{latexsym}

\usepackage[T1]{fontenc}

\usepackage[utf8]{inputenc}

\usepackage{microtype}

\usepackage{graphicx}
\usepackage{dsfont}
\usepackage{amsmath,amsfonts,bm,mathtools,amssymb}
\usepackage{cleveref}
\usepackage{dsfont}
\usepackage{caption}
\usepackage{subcaption}
\usepackage{xcolor}
\usepackage{soul}
\usepackage{hyperref}
\usepackage{adjustbox}
\usepackage{booktabs}
\usepackage{makecell}
\usepackage{nicefrac}
\usepackage{xspace}
\usepackage{amsfonts}
\usepackage{amsthm}
\usepackage{amsmath}
\usepackage{amssymb}
\usepackage{thmtools} 
\usepackage{thm-restate}
\usepackage{scalerel}
\usepackage{float}
\usepackage{flafter}
\usepackage{svg}
\usepackage{mathrsfs}

\usepackage[textsize=tiny,disable]{todonotes}

\newtheorem{defin}{Definition}

\usepackage{cleveref}
\crefname{section}{\S}{\S\S}
\Crefname{section}{\S}{\S\S}
\crefname{table}{Tab.}{Tabs.}
\crefname{figure}{Fig.}{Figs.}
\crefname{algorithm}{Alg.}{}
\crefname{equation}{Eq.}{Eqs.}
\crefname{appendix}{App.}{}
\crefname{theorem}{Theorem}{}
\crefname{proposition}{Proposition}{}
\crefname{defin}{Definition}{}
\crefname{cor}{Corollary}{}
\crefname{observation}{Observation}{}
\crefname{assumption}{Assumption}{}
\crefformat{section}{\S#2#1#3}
\crefformat{footnote}{#2\footnotemark[#1]#3}

\title{Towards a Similarity-adjusted Surprisal Theory}

\author{
    Clara Meister \ \ \ \ \ Mario Giulianelli \ \ \ \ \ Tiago Pimentel \\
    ETH Zürich, Department of Computer Science, Institute for Machine Learning\\
    \texttt{\{\href{clara.meister@inf.ethz.ch}{clara.meister},\href{mario.giulianelli@inf.ethz.ch}{mario.giulianelli},\href{tiago.pimentel@inf.ethz.ch}{tiago.pimentel}\}@inf.ethz.ch}\\
    \setlength{\fboxsep}{2.5pt}%
    \setlength{\fboxrule}{2.5pt}%
    \fcolorbox{white}{white}{
        \includesvg[width=.15\linewidth]{ethz-logo}
    }
}

\begin{document}
\maketitle

\newcommand{\dist}{d}
\newcommand{\topic}{t}
\newcommand{\entropy}{\mathrm{H}}
\newcommand{\defeq}{\mathrel{\stackrel{\textnormal{\tiny def}}{=}}}
\newcommand{\dataset}{{\cal D}}

\definecolor{mygray}{rgb}{.3, .3, .3}
\definecolor{mygreen}{rgb}{0, .5, 0}
\definecolor{myred}{rgb}{.6, 0.15, 0.15}

\newcommand{\coloursim}{ForestGreen}
\newcommand{\colourmeaning}{blue}
\newcommand{\colourcontext}{orange}
\newcommand{\colourdist}{purple}

\newcommand{\mywordfunc}[2]{\newcommand{#1}{#2}}
\newcommand{\mymeaningfunc}[2]{\newcommand{#1}{#2}}
\newcommand{\mycontextfunc}[2]{\newcommand{#1}{#2}}
\newcommand{\mydistfunc}[2]{\newcommand{#1}{{\color{\colourdist}#2}}}
\newcommand{\mysimfunc}[2]{\newcommand{#1}{{\color{\coloursim}#2}}}
\newcommand{\newtext}[1]{{\color{blue} #1}}

\newcommand{\addcites}{{\color{red}(add cites)}\xspace}
\newcommand{\writemore}{{\color{red}...(write more)}\xspace}
\mywordfunc{\word}{w}
\mywordfunc{\Word}{W}
\mycontextfunc{\words}{\boldsymbol{w}}
\mycontextfunc{\Words}{\boldsymbol{W}}
\newcommand{\vocab}{\mathcal{V}}
\newcommand{\calR}{\mathcal{R}}
\newcommand{\surp}{\mathrm{h}}
\newcommand{\defn}[1]{\textbf{#1}}
\mymeaningfunc{\meaningvar}{m}
\mymeaningfunc{\Meaningvar}{\mathrm{M}}
\newcommand{\meaningspace}{\mathcal{M}}
\newcommand{\costfunc}{\mathrm{cost}}
\newcommand{\kl}{\mathrm{KL}}
\newcommand{\llh}{\mathcal{L}}
\newcommand{\deltallh}{\Delta_{\llh}}

\mydistfunc{\distfuncbase}{d}
\newcommand{\infovalue}{\mathrm{i}_{\distfuncbase}}
\newcommand{\distfunc}{\distfuncbase_{\words_{<t}}}
\mydistfunc{\distfuncvecbase}{\boldsymbol{d}}
\newcommand{\distfuncvec}{\distfuncvecbase_{\words_{<t}}}
\mydistfunc{\distmatrixbase}{\mathrm{D}}
\newcommand{\distmatrix}{\distmatrixbase_{\words_{<t}}}
\mysimfunc{\similarityfuncbase}{z}
\newcommand{\similarityfunc}{\similarityfuncbase_{\words_{<t}}}
\newcommand{\probvector}{\mathbf{p}}
\mysimfunc{\similarityvectoruncontext}{\mathbf{z}}
\newcommand{\similarityvector}{\similarityvectoruncontext_{\word_t}}
\mysimfunc{\similaritymatrix}{\mathrm{Z}}
\newcommand{\embed}{\phi}
\newcommand{\alternativehyp}{\mathcal{A}_{\words_{<t}}}
\newcommand{\diverseent}{\entropy_{\similarityfuncbase}}
\newcommand{\diversesurp}{\surp_{\similarityfuncbase}}
\newcommand{\diversesurpalpha}{\surp_{\similarityfuncbase^{\alpha}}}
\newcommand{\diversekl}{\kl_{\similarityfuncbase}}
\newcommand{\diversecostfunc}{\costfunc_{\similarityfuncbase}}
\newcommand{\pz}{p_{\similarityfuncbase}}

\newcommand{\mathcomment}[1]{\text{\color{gray}#1}}

\newcommand{\yhat}{\widehat{y}}
\newcommand{\vpsi}{\boldsymbol{\psi}}
\newcommand{\regressparams}{\vpsi}
\newcommand{\model}{f_{\vpsi}}
\newcommand{\modelone}{f_{\vpsi_1}}
\newcommand{\modeltwo}{f_{\vpsi_2}}
\newcommand{\pmodel}{p_{\vpsi}}
\newcommand{\vars}{\mathbf{x}}
\newcommand{\vtheta}{\boldsymbol{\theta}}
\newcommand{\gptsmall}{GPT-2 \texttt{small}\xspace}

\newcommand{\ptheta}{p_{\scaleto{\vtheta}{4pt}}}

\newcommand{\diversesurprisal}{similarity-adjusted surprisal\xspace}
\newcommand{\DiverseSurprisal}{Similarity-adjusted Surprisal\xspace}
\newcommand{\Diversesurprisal}{Similarity-adjusted surprisal\xspace}
\newcommand{\diverse}{similarity-adjusted\xspace}
\newcommand{\Diverse}{Similarity-adjusted\xspace}

\begin{abstract}
Surprisal theory posits that the cognitive effort required to comprehend a word is determined by its contextual predictability, quantified as surprisal. 
Traditionally, surprisal theory treats words as distinct entities, overlooking any potential similarity between them.
\citet{giulianelli-etal-2023-information} address this limitation by introducing information value, a measure of predictability designed to account for similarities between communicative units. 
Our work leverages \citeposs{ricotta2006towards} diversity index to extend surprisal into a metric that we term \diversesurprisal, exposing a mathematical relationship between surprisal and information value. 
\Diversesurprisal aligns with information value when considering graded similarities and reduces to standard surprisal when words are treated as distinct. Experimental results with reading time data indicate that \diversesurprisal adds predictive power beyond standard surprisal for certain datasets, suggesting it serves as a complementary measure of comprehension effort.
\end{abstract}

\section{Introduction}

Surprisal theory \cite{hale2001probabilistic} states that the effort a reader must spend to comprehend a word is a function of its contextual predictability, which is typically quantified as its surprisal, or negative log-probability. 
With numerous empirical \citep[\textit{inter alia}]{smith2008optimal,smith2013-log-reading-time,shain-2019-large,shain-2021-cdrnn} and theoretical \citep{levy2008expectation} studies supporting its claims, surprisal theory is a widely-accepted model of the effort required 
for sentence comprehension. 
Notably, surprisal theory treats words as completely distinct from one another, disregarding that they may express similar meanings. 
Motivated by this shortcoming, \citet{giulianelli-etal-2023-information} proposed a new measure of comprehension effort: \defn{information value}.
Similarly to surprisal, information value quantifies the predictability of a linguistic unit in context;
unlike surprisal, however, it accounts for 
communicative equivalences between possible continuations.
\citeauthor{giulianelli-etal-2023-information} find this metric to be 
a significant predictor of utterance-level reading times and acceptability judgments, both independently and in addition to surprisal.

Similarly inspired, we investigate \defn{\diversesurprisal} as a potential measure of comprehension effort. 
This measure is a natural extension of \citeposs{ricotta2006towards} diversity index---which itself is a generalization of Shannon's entropy used to measure species biodiversity. 
Given a choice of similarity function, \diversesurprisal computes a word's predictability while considering its likeness to alternative continuations. 
Through this measure, we connect information value and standard surprisal, showing a mathematical relationship between these two metrics: \diversesurprisal has a monotonically increasing relationship with information value and reverts to standard surprisal when its similarity function regards different words as completely distinct.\looseness-1

In experiments with reading time data, we explore the psycholinguistic predictive power of \diversesurprisal with semantic, syntactic and orthographic notions of word distance. 
For some datasets, we see that---as with information value---\diversesurprisal provides significant predictive power above and beyond standard surprisal. 
This complementarity suggests that there are aspects of incremental comprehension effort that are not fully captured by the classic definitions used in surprisal theory. 
We also observe that non-contextual notions of semantic distance lead to better predictors than using contextual notions of distance, which supports observations that incremental comprehension effort is (at least partially) driven by shallow semantic processing.\looseness=-1

\section{Background}
This section presents surprisal and information value, providing the relevant formal background for our \diversesurprisal{}.
We will use $\vocab$ to refer to the vocabulary, i.e., a finite, non-empty set of words, and $\vocab^{*}$ to refer to the set of all strings formed by concatenating words in $\vocab$.
We denote words as $\word \in \vocab$, and strings (sequences of words) as $\words \in \vocab^{*}$.  An index $t$, e.g., $\word_t$ and $\words_{<t}$, is used to mark positions within a string.

\subsection{Surprisal Theory}

According to surprisal theory, 
comprehending a word $\word_t \in \vocab$ in its context $\words_{<t} \in \vocab^*$ requires a reader to update their beliefs about the intended meaning of the sentence,
performing probabilistic inference over the space of possible meanings \cite{hale2001probabilistic,levy2008expectation}.
The cost of this belief update is equal to the word's \defn{surprisal}, or information content, whose formal definition is:\footnote{See \cref{app:cost} for a derivation and discussion of the relationship between $\surp(\word_t)$ and processing cost.}\looseness=-1
\begin{align}\label{eq:surprisal}
    \surp(\word_t) \defeq -\log p(\word_t \mid \words_{<t})
\end{align}
If surprisal theory provides an accurate account of sentence comprehension, then we should find traces of this online inferential process
in humans' behavioral responses to language comprehension tasks.
In particular, assuming that a word's processing cost is reflected in its \defn{reading time} (RT), a word's RT should be an increasing function
of its surprisal \citep{smith2008optimal}. 
A large body of empirical work has examined the relationship between RT and surprisal, with results supporting surprisal theory \citep[][\emph{inter alia}]{smith2013-log-reading-time,wilcox2023testing,shain2024large}.
Given these established results, new predictors of reading behavior, such as information value, would benefit from grounding in surprisal theory.\looseness=-1

\subsection{Information Value}

\newcommand{\powerset}{\mathcal{P}}
\newcommand{\R}{\mathbb{R}}
\newcommand{\N}{\mathbb{N}}

\citet{giulianelli-etal-2023-information} recently introduced a new measure to predict the cost associated with reading: \defn{information value}. 
Let $\alternativehyp \mathop{\in} \mathscr{P}({\vocab^{*}})$ be a multiset of plausible alternative continuations that a reader may expect to follow a given context $\words_{<t}$.\footnote{$\mathscr{P}({\vocab^{*}})$ is 
the set of all multisets of elements in $\vocab^{*}$.\looseness-1}
The information value of a continuation $\words_{\geq t} \mathop{\in} \vocab^{*}$ 
is defined as how different it is from continuations in $\alternativehyp$.
If $\words_{\geq t}$ is similar to what a reader expects, i.e., to elements of $\alternativehyp$, then it does not convey much information, and should thus require little effort to process. If $\words_{\geq t}$ differs greatly from expected continuations, then it conveys more information and its processing cost is higher.
Formally, we write $\words_{\geq t}$'s information value as
$\distfuncvec(\words_{\geq t}, \alternativehyp)$, where $\distfuncvec: \vocab^{*} \times \mathscr{P}(\vocab^{*}) \to \R_{+}$ is a context-conditioned distance function.\footnote{This function may also be constant in $\words_{<t}$, for example, if $\distfunc$ measures orthographic distance between different $\words_{\geq t}$.\looseness=-1}

Following \citeauthor{giulianelli-etal-2023-information}, we consider alternative sets $\alternativehyp$
whose elements are sampled independently from $p(\cdot \!\mid \words_{<t})$
and distance functions $\distfuncvec$ which apply element-wise to each instance in the alternative set through $\distfunc: \vocab^* \times \vocab^* \to \R_{+}$; we then aggregate individual distances by taking their mean. 
However, to make the comparison to standard next-word surprisal more natural, we take $\alternativehyp$ to be composed of individual words rather than full string continuations.
We will thus use the notation:
\begin{align}
    &\distfuncvec(\word, \alternativehyp) = \frac{\sum_{\word' \in \alternativehyp} \distfunc(\word, \word')}{|\alternativehyp|}  \label{eq:mc}
\end{align}
This is a Monte Carlo estimator of the expected distance of a word $\word$ from other next words that start continuations of $\words_{<t}$ \citep{giulianelli-etal-2024-generalized}. We refer to it as \defn{next-word information value}:
\begin{align}\label{eq:expectation}
        \infovalue(\word_t) &\defeq 
    \sum_{\word' \in \vocab} \distfunc(\word_t, \word')\, p(\word' \mid \words_{<t})
\end{align}

\section{\DiverseSurprisal Theory}

To bridge the theoretical gap between surprisal and information value,
we wish to derive a notion of per-word information content that accounts for similarities between different plausible continuations, rather than treating words as completely distinct outcomes.
To this end, we turn to diversity indices: metrics developed in the field of biology to quantify biodiversity.\footnote{For an overview, see \citet{leinster2012measuring}.} Analogous to our setting, when quantifying biodiversity, it is desirable to have a metric that takes into account that some species are more closely related to each other (e.g., two species in the same genus vs.\ in different genera). We adapt one of these metrics to our context.

\subsection{\Diverse Entropy and Surprisal}

Let $R$ be a categorical random variable that takes on values $r\in\calR$. 
Further, let $\similarityfuncbase: \calR \mathop{\times} \calR \mathop{\to} [0, 1]$ be a \defn{similarity function}; it is 0 when $r$ and $r'$ are completely dissimilar and 1 when they are equivalent. 
\citeposs{ricotta2006towards} diversity index is then defined as: 
\begin{align}
    \diverseent(R) \defeq\label{eq:div_ent} -\sum_{r\in \calR} p(r) \log \sum_{r'\in \calR} \similarityfuncbase(r, r') p(r')
\end{align}
If we use an identity similarity function such that $\similarityfuncbase(r, r')\mathop{=}1$ if $r\mathop{=}r'$ and $0$ otherwise, then \cref{eq:div_ent} is equivalent to Shannon's entropy \cite{shannon1948mathematical}. We thus refer to \cref{eq:div_ent} as \defn{\diverse entropy}.\looseness-1

Bearing in mind entropy's close mathematical relationship to surprisal (i.e., entropy is the expected value of surprisal), \diverse entropy can be extended to a notion of surprisal that accounts for similarities between classes. 
We define the \defn{\diversesurprisal
} of outcome $r$ as: 
\begin{align}\label{eq:div_surp_base}
    \diversesurp(r) \defeq - \log \sum_{r' \in \calR} \similarityfuncbase(r, r')\, p(r')
\end{align} 
Comparably to \cref{eq:div_ent}, when the identity similarity function is used, then $\diversesurp(r) = - \log p(r) = \surp(r)$. 
While a variant of \diversesurprisal has been used for measuring semantic uncertainty in neural machine translation \cite{cheng-vlachos-2024-measuring}, its application in psycholinguistics has not yet been explored.

\subsection{\DiverseSurprisal and Information Value}

\newcommand{\monotonicfunc}{\overset{\propto}{\nearrow}}

Now let $\similarityfunc: \vocab \times \vocab \to [0, 1]$ be a similarity function that measures how similar two words are in context $\words_{<t}$. 
By using $\similarityfunc$ in \cref{eq:div_surp_base}, we arrive at a notion of \diversesurprisal for a word in context.
\begin{defin} 
    The \diversesurprisal $ \diversesurp$ of a word $\word_t$ in context $\words_{<t}$ is defined as:
    \begin{align}\label{eq:div_surp}
        \diversesurp(\word_t) &\mathop{\defeq}
        - \log \!\! \sum_{\word' \in \vocab} \!\! \similarityfunc(\word_t, \word')\, p(\word' \mathop{\mid} \words_{<t})
    \end{align}
\end{defin}
Given this definition, we can now present the main theoretical result of this paper.

\begin{restatable}{theorem}{semsurpvsinfovalue} 
\label{theorem:semsurp_vs_infovalue}
\!Let $\distfunc\!\!:\!\vocab\times\!\vocab\!\to\![0, 1]$ 
and $\similarityfunc (\word_t, \word')\!=\!1 -\distfunc (\word_t, \word')$.  
Under these settings, next-word information value and \diversesurprisal have a monotonic, strictly increasing relationship. 
%
\end{restatable}
\begin{proof}
    Proof in \cref{app:semsurp_vs_infovalue}.
\end{proof}
\noindent Note that this result trivially extends to information value and \diversesurprisal measured over finite strings $\words_{\geq t} \mathop{\in} \vocab^{*}$ of arbitrary length.

Because standard surprisal is a special case of \diversesurprisal, this result connects surprisal and information value; \citeauthor{giulianelli-etal-2023-information}'s findings can thus be seen as supporting an enriched notion of surprisal theory. 
\Cref{app:proof_semanticsurprisaltheorycost} shows how to use this relationship to derive a \diverse definition of processing cost.\looseness=-1

\subsection{Related Theories in Psycholinguistics}
Several prior works in psycholinguistics have also examined variants of standard surprisal, e.g., decomposing \cite{roark-etal-2009-deriving,li-2023}, augmenting \cite{aurnhammer-2021}, or revising it \cite{arehalli-etal-2022-syntactic,giulianelli-etal-2023-information,giulianelli-etal-2024-generalized,giulianelli-etal-2024-incremental}. 
Some of these are motivated by the belief that the language comprehension process can be broken down into multiple distinct subtasks, for which there are different processing mechanisms \cite{Kuperberg-2016}; they then associate 
variants of surprisal with these different cognitive processes. \citet{roark-etal-2009-deriving}, for instance, proposes that surprisal can be decomposed into a syntactic and a semantic component, each associated with its own cognitive process. 
In contrast to some of these works---for a subset of which we provide more detailed descriptions in \cref{app:related_work}---we do not propose a decomposition of or alternative to surprisal theory. 
Rather, we view our work as offering a revised mathematical definition of surprisal, but still within the original surprisal theory framework.\looseness=-1

\section{Experimental Methodology}

\newcommand{\yn}{y_n^{(i)}}
\newcommand{\ybar}{\bar{y}}
\newcommand{\ybarn}{\ybar_n}
\newcommand{\varsn}{\vars_n}


\subsection{Data}

We consider four datasets of naturalistic reading: Brown \citep{smith2013-log-reading-time}, Dundee \citep{dundee}, Natural Stories \citep{futrell-etal-2018-natural}, and Provo \citep{provo}. 
To collect these datasets, participants were administered text passages to read, and the time they spent fixating on each word was recorded. More details are provided in \cref{app:exp_setup}. 
We organize these measurements into data points consisting of $\langle\varsn, \yn\rangle$ pairs, where $\yn \in \R_{+}$ is participant $i$'s RT of word $\word_n$, and $\varsn \in \R^{d}$ are word $\word_n$'s characteristics. 
These characteristics---which we refer to as predictors---consist of quantities such as a word's surprisal or unigram frequency.
Following prior work \citep[][\emph{inter alia}]{wilcox2020predictive,meister-etal-2021-revisiting}, we average RTs across participants, resulting in a single mean RT per word, $\ybarn$. Our models are trained and tested to predict these averages.

\subsection{Reading Time Regressors}

Let $\model$ be a function that takes $\varsn$ and predicts~$\ybarn$.
To avoid overlap in terminology with our discussion of language models, we refer to $\model$ as a regressor, and denote its parameters as $\regressparams$.
A regressor $\model$ can take different functional forms. In light of prior work showing the surprisal--RT relationship to be largely linear \citep{smith2008optimal,smith2013-log-reading-time,wilcox2023testing,shain2024large}, we restrict ourselves to linear $\model$.
Given a trained regressor $\model$ and a new data point $\vars$, we can use $\model$ to either predict $\yhat = \model(\vars)$ or to estimate the probability of observing a specific $\ybar$ given an $\vars$: $\pmodel(\ybar \mathop{\mid} \vars) \mathop{=} \frac{(\ybar \mathop{-} \model(\vars))^2}{\sigma^2}$, where 
$\sigma^2$ is the regressor's variance
estimated on the training set.
The log-likelihood $\llh(\model, \dataset)$ of a dataset $\dataset$ under $\model$ is then given by the (log of the) joint probability of observing those data points according to $\model$.

\begin{table*}[t]
\centering
\resizebox{\textwidth}{!}{
\begin{tabular}{lcccccccc}
\toprule
   & \multicolumn{4}{c}{\textbf{\Diversesurprisal}} & \multicolumn{4}{c}{\textbf{Information value}} \\
   \cmidrule(lr){2-5} \cmidrule(lr){6-9}
   & Non-contextual & Contextual & POS & Orthographic & Non-contextual & Contextual & POS & Orthographic
     \\
\midrule

Brown & {-0.02}\phantom{$^{***}$}& {-0.04}\phantom{$^{***}$}& \phantom{-}{1.83}\phantom{$^{***}$}& \phantom{-}{0.46}\phantom{$^{***}$}& {-0.01}\phantom{$^{***}$}& {-0.04}\phantom{$^{***}$}& \phantom{-}{0.07}\phantom{$^{***}$}& \phantom{-}{0.21}\phantom{$^{***}$} \\
Dundee & \phantom{-}\textcolor{mygreen}{0.13}$^{***}$& \phantom{-}{0.00}\phantom{$^{***}$}& \phantom{-}{0.35}\phantom{$^{***}$}& \phantom{-}{0.24}\phantom{$^{***}$}& \phantom{-}\textcolor{mygreen}{0.14}$^{***}$& \phantom{-}{0.00}\phantom{$^{***}$}& \phantom{-}{0.01}\phantom{$^{***}$}& \phantom{-}{0.02}\phantom{$^{***}$} \\
Natural Stories & \phantom{-}\textcolor{mygreen}{0.50}$^{***}$& \phantom{-}\textcolor{mygreen}{0.32}$^{*}$\phantom{$^{**}$}& \phantom{-}{0.57}\phantom{$^{***}$}& {-0.03}\phantom{$^{***}$}& \phantom{-}\textcolor{mygreen}{0.58}$^{***}$& \phantom{-}\textcolor{mygreen}{0.32}$^{*}$\phantom{$^{**}$}& \phantom{-}{0.04}\phantom{$^{***}$}& \phantom{-}{0.07}\phantom{$^{***}$} \\
Provo & \textcolor{myred}{-0.18}$^{***}$& \phantom{-}{0.04}\phantom{$^{***}$}& \phantom{-}{0.86}\phantom{$^{***}$}& {-0.15}\phantom{$^{***}$}& \textcolor{myred}{-0.19}$^{**}$\phantom{$^{*}$}& \phantom{-}{0.02}\phantom{$^{***}$}& {-0.01}\phantom{$^{***}$}& {-0.21}\phantom{$^{***}$} \\
\bottomrule
\end{tabular}}%
\vspace{-5pt}
\caption{$\deltallh$ (in $10^{-2}$ nats) over baseline \emph{with} surprisal terms when adding a \diversesurprisal or information value term for each of the current and previous 3 words. Monte Carlo (MC) estimation with 50 samples.}
\label{tab:delta_llh_all}
\vspace{-8pt}
\end{table*}

\subsection{Predictors}\label{sec:predictors} 

In all of our experiments, predictors $\varsn$ include two baseline variables typically used in RT analyses: word length (measured in characters) and unigram frequency.   
To account for spillover effects, which are caused by continued processing of previous words~\cite{just1982paradigms,frank2013reading},
we include in $\varsn$ these variables for the current word $\word_n$, as well as for the three words preceding it.\looseness=-1

For our information-theoretic predictors, we use estimators. 
Our estimation of \cref{eq:surprisal,eq:mc,eq:div_surp} can be summarized as: i) we replace the distribution $p$ with a parameterized language model $\ptheta$---specifically \gptsmall---when computing \cref{eq:surprisal,eq:expectation,eq:div_surp};\footnote{This is standard practice in psycholinguistics \citep[\textit{inter alia}]{smith2008optimal,goodkind-bicknell-2018-predictive,wilcox2020predictive}. 
}
ii) when it is too computationally expensive to sum over the entire vocabulary---which is required for exactly computing the expectations in \cref{eq:expectation,eq:div_surp}---we use a Monte Carlo estimator (with 50 samples) for \diversesurprisal and next-word information value. 
More details on these information-theoretic estimators, as well as on methods for estimating unigram frequencies, are provided in \cref{app:exp_setup}.\looseness=-1

\subsection{Similarity Functions}
For both \diversesurprisal and information value, we consider several similarity functions;
for each similarity function, we define an analogous distance as $\distfunc\!(\word,\! \word')\!=\!1\!-\!\similarityfunc\!(\word,\! \word')$. 
Details about precise estimation procedures are in \cref{app:exp_setup}. 
\paragraph{Word Embedding Similarity. }
Let $\embed: \vocab \rightarrow \mathbb{R}^d$ be a word embedding function, which may or not depend on context $\words_{<t}$.
We compute the similarity between $\word$ and $\word'$ as the normalized cosine similarity:\looseness=-1
\begin{equation}\label{eq:sem_sim}
    \similarityfunc(\word, \word') \mathop{=} \frac{1}{2} \!\! \left(\frac{\embed(\word)\cdot \embed(\word')}{||\embed(\word)||\,||\embed(\word')||} \mathop{+} 1\!\right)
\end{equation}
When computed using word embeddings, cosine similarity has proven itself a good metric of semantic similarity \cite{erk-2009-representing,pennington-etal-2014-glove}. 
We again use \gptsmall{} in all of our experiments as $\embed$ to produce non-contextual and contextual word embeddings.

\paragraph{Part-of-Speech Similarity.}
We use a measure of part-of-speech (POS) similarity as a notion of syntactic similarity: 
\begin{equation}
    \similarityfunc(\word, \word') \mathop{=}
\begin{cases}
    1, &\!\! \text{if}\, \textsc{pos}_{\words_{<t}}(\word) = \textsc{pos}_{\words_{<t}}(\word' )\\
    0, & \!\!\text{otherwise}
\end{cases}
\end{equation}
where $\textsc{pos}_{\words_{<t}}$ is a POS-tagging model. We use the \texttt{pos-fast} model of the \texttt{flair} library.\footnote{\url{https://github.com/flairNLP/flair}}

\paragraph{Orthographic Similarity.}
We further use a normalized version of string edit (Levenshtein) distance to quantify orthographic similarity. Let $\distfuncbase_{\textsc{l}}(\word,\word')$ be the edit distance between $\word$ and $\word'$. Our orthographic similarity function is then
\begin{equation}
    \similarityfunc(\word, \word') \mathop{=} 1 - 
 \frac{\distfuncbase_{\textsc{l}}(\word,\word')}{\max\{|\word|,|\word'|\}}
\end{equation}
where $|\cdot|$ measures string length in characters.

\subsection{Evaluation}
\newcommand{\testdata}{\dataset_{\scaleto{\mathrm{test}}{4pt}}}
We quantify the predictive power of a predictor as the change in log-likelihood ($\deltallh$) of held-out data points $\testdata$ between a regressor $\modelone$ that includes this predictor and another, $\modeltwo$, that does not:
\begin{equation}
    \deltallh = \llh(\modelone, \testdata) - \llh(\modeltwo, \testdata)~.
\end{equation}
This is a standard measure of predictive power in psycholinguistics \cite[][]{goodkind-bicknell-2018-predictive,wilcox2020predictive}.
We estimate $\deltallh$ via 10-fold cross-validation: we use 9 data folds at a time to estimate the parameters of $\modelone$ and $\modeltwo$ and compute $\deltallh$ on the 10$^\text{th}$ fold;
we report mean $\deltallh$ across folds. We run paired permutation tests with these 10-fold results to evaluate statistical significance.\looseness=-1

\section{Results and Discussion}

\paragraph{Predictive power of \diversesurprisal.}
We evaluate the psycholinguistic predictive power that \diversesurprisal and information value provide beyond standard surprisal. 
To this end, we compute $\deltallh$ (\cref{tab:delta_llh_all}) when adding these predictors to a regressor that already includes surprisal.  
In Natural Stories, we find \diversesurprisal and information value provide predictive power complementary to standard surprisal when using embedding-based (both contextual and non-contextual) similarity functions.
In Dundee, significant additional predictive power only results from using non-contextual embedding similarities.
Meanwhile, in Provo and Brown, \diversesurprisal and information value do not add predictive power beyond surprisal; they are significant predictors when evaluated against a control baseline, but less so than surprisal (see \cref{tab:delta_llh_control_baseline,tab:delta_llh_replace_surprisal} in \cref{app:add_results}). In \cref{app:add_results}, we also present results when 
exponentiating our definition of word pair similarity in \cref{eq:sem_sim} by an $\alpha$ such that \diversesurprisal converges to standard surprisal as $\alpha \rightarrow \infty$; we find that \diversesurprisal's predictive power is not strongly influenced by such choice.

The differences in predictive power of the two \diverse measures follow other notable trends across datasets. 
Predictive power is lowest on Provo, where stimuli have an average of only 50 words each, followed by Brown (553 words); it is highest for Natural Stories and Dundee, both containing stimuli whose average lengths are above 1000 words. 
These results suggest equipping surprisal with semantic measures of word predictability is beneficial when the psycholinguistic measurements at hand are collected for stimuli situated in broader discourse contexts. 
Other factors, such as the texts' style or topic, may have also played a role in these differences across datasets.

\paragraph{Broader implications.}
There is also an interpretation of these results as corroborating recently proposed theories of language comprehension.  
The Natural Stories corpus contains low-frequency (albeit still grammatically correct) syntactic constructions. 
Thus, in this corpus, we encounter continuations that are less predictable from the context but have high similarity with more predictable alternative continuations. 
The result that our semantic variants of \diversesurprisal and information value provide significant predictive power over standard surprisal, particularly for this dataset, can be taken as support for models of heuristic processing (e.g., \citealp{li-2023}; see \cref{app:related_work} for further discussion). 
Similarly, the overall higher predictive power provided by non-contextual similarity functions (when compared to contextual ones) could be taken as evidence that incremental comprehension effort is more sensitive to forms of shallow semantic processing \cite{barton1993case,daneman2007shallow} than to deep integration of contextualized word meaning.
However, due to the known sensitivity of contextual embeddings to word-unspecific sentential information \cite{klafka-ettinger-2020-spying,erk2022word}, 
further analysis with semantic similarity functions is required to confirm this finding.

\section{Conclusion}
This work introduces \diversesurprisal: 
a measure of contextual word predictability that takes into account word similarities.
By equipping surprisal with the ability to account for words' relationships, we reconcile surprisal theory's predictions with those of information value and demonstrate their mathematical relationship. 
Our experimental results on RT data indicate \diversesurprisal has predictive power beyond that of standard surprisal, thus validating and enriching surprisal theory. 
Points for future research include analyzing similarity functions that capture different characteristics of word meaning, as well as measuring the predictive power of \diversesurprisal for other indices of processing difficulty, such as N400 and other event-related brain potentials.\looseness=-1

\section*{Limitations}
We do not provide a comprehensive assessment of different design choices and experimental settings, limiting the definitiveness with which we can draw conclusions about the efficacy of \diversesurprisal as a predictor of language comprehension. These different choices and settings deserve further exploration. 
The reading time datasets that we employ are in English, and thus, we can only draw conclusions about reading behavior in the English language.  
Further, we only consider three functions for computing word similarities. 
As the functional form of the surprisal--reading time relationship has proven to be quite important for the psycholinguistic predictive power of surprisal, it is conceivable that the choice of similarity function could likewise have a large impact on the psycholinguistic predictive power of our diverse predictors. 

\section*{Acknowledgments}
Clara Meister was supported by a Google PhD Fellowship. 
Mario Giulianelli was supported by an ETH Zurich Postdoctoral Fellowship. 
We thank our anonymous reviewers for their insightful feedback and helpful pointers to related works.

\bibliography{custom}

\appendix
\onecolumn

\section{Relationship between Information Value and \DiverseSurprisal}
\label{app:semsurp_vs_infovalue}
\semsurpvsinfovalue*
\vspace{-15pt}
\begin{align}
    \infovalue(\word_t) \monotonicfunc \diversesurp(\word_t)
\end{align}
\begin{proof}
    Using the relationship $\distfunc(\word_t, \word') = 1 - \similarityfunc(\word_t, \word')$, simple algebraic manipulation shows that:
    \begin{align}
        \infovalue(\word_t) 
        &= \sum_{\word' \in \vocab} p(\word' \mid \words_{<t})\, \distfunc(\word_t, \word') \nonumber \\
        &= \sum_{\word' \in \vocab} p(\word' \mid \words_{<t})\, \left(1 -  \similarityfunc(\word_t, \word')\right) \nonumber \\
        &= 1 - \sum_{\word' \in \vocab} p(\word' \mid \words_{<t})\, \similarityfunc(\word_t, \word') \nonumber \\
        &\monotonicfunc - \log \sum_{\word' \in \vocab} p(\word' \mid \words_{<t})\, \similarityfunc(\word_t, \word') \nonumber \\
        &= \diversesurp(\word_t \mid \words_{<t})
    \end{align}
    where $\monotonicfunc$ indicates a monotonic, strictly increasing relationship and follows from the fact that $\log$ is a monotone, strictly increasing function.
\end{proof}

\section{The Mathematical Relationship between Processing Costs and Surprisal}\label{app:cost}
\subsection{Equivalence between Surprisal and KL Divergence}

A word's surprisal is equivalent to the Kullback--Leibler (KL) divergence between two probability distributions over a sentence's potential meanings: one with and one without knowledge of that word \citep{levy2008expectation}.  
Formally, let $\meaningspace$ be the space of potential sentence meanings, let $\meaningvar \in \meaningspace$ be a meaning and let $p(\meaningvar \mid \words_{<t})$ be the probability of meaning $\meaningvar\in\meaningspace$ conditioned on a prefix $\words_{<t}$. 
We can define the cost of reading a word $\word_t$ as the amount of energy spent to update this distribution. 
Assuming this energy consumption is a function of the distance between the prior and posterior distributions over meanings after observing $\word_t$, we can define cost according to the $\kl$ divergence between these two distributions, a standard measure of the difference between distributions:
\begin{align}
\costfunc(\word_t)&\defeq \kl(p(\meaningvar \mid \words_{< t} \circ \word_t) \mid\mid p(\meaningvar \mid \words_{<t})) \nonumber \\
    &= \sum_{\meaningvar \in \meaningspace} p(\meaningvar, \words_{< t} \circ \word_t) \log \frac{p(\meaningvar \mid \words_{< t} \circ \word_t)}{p(\meaningvar \mid \words_{< t})}~.
\end{align}
Under standard assumptions,\footnote{
We assume, as in \citet{levy2008expectation}, that $p(\words_{<t} \circ \word_t \mid \meaningvar)$ is deterministic, i.e., there is only one sequence $\words_{<t} \circ \word_t$ which can be used to convey each meaning $\meaningvar$. While perhaps unrealistic, we can still draw insights from this result. \label{footnote}} 
we can show that this divergence is equivalent to word $\word_t$'s surprisal.

\begin{restatable}{theorem}{surprisaltheorycost} 
\label{theorem:surprisaltheorycost}
\textbf{Cost equals information content} \citep[result from][reiterated here in the notation of this paper]{levy2008expectation}.
Under standard assumptions about $p(\meaningvar, \words_{<t} \circ \word_t)$, we can show that:
\begin{align}
    \costfunc(\word_t) = \surp(\word_t)~.
\end{align}
\end{restatable}
\begin{proof}
Let $p(\words_{<t} \circ \word_t \mid \meaningvar)$ be deterministic, i.e., there is only one sequence $\words_{<t} \circ \word_t$ which can be used to convey each meaning $\meaningvar$. 
From Bayes theorem, we then have that:
\begin{subequations}
\begin{align}
    p(\meaningvar \mid \words_{< t} \circ \word_t) 
    &= \frac{\overbrace{p(\word_t \mid \meaningvar, \words_{< t})}^{\text{$=1$ because deterministic}}\,p(\meaningvar \mid \words_{< t})}{p(\word_t \mid \words_{< t})} \\
    &= \frac{p(\meaningvar \mid \words_{< t})}{p(\word_t \mid \words_{< t})}
\end{align}
\end{subequations}
We now use this equality to arrive at the desired result:
\begin{subequations}
    \begin{align}
        \costfunc(\word_t) 
        &= \kl(p(\meaningvar \mid \words_{< t} \circ \word_t) \mid\mid p(\meaningvar \mid \words_{<t})) \\
        &= \sum_{\meaningvar \in \meaningspace} p(\meaningvar, \words_{< t} \circ \word_t) \log \frac{p(\meaningvar \mid \words_{< t} \circ \word_t)}{p(\meaningvar \mid \words_{< t})} \\
        &= \sum_{\meaningvar \in \meaningspace} p(\meaningvar, \words_{< t} \circ \word_t) \log \frac{p(\meaningvar \mid \words_{< t})}{p(\word_t \mid \words_{< t})\,p(\meaningvar \mid \words_{< t})} \\
        &= \sum_{\meaningvar \in \meaningspace} p(\meaningvar, \words_{< t} \circ \word_t) \log \frac{1}{p(\word_t \mid \words_{< t})} \\
        &= \log \frac{1}{p(\word_t \mid \words_{< t})} \\
        &= \surp(\word_t)
    \end{align}
\end{subequations}
\end{proof}

\subsection{\DiverseSurprisal as a KL Divergence}
\label{app:proof_semanticsurprisaltheorycost}
In this section, we first define a number of similarity-aware distributions as:
\begin{subequations} \label{eq:pz_definitions}
\begin{align}
    &\pz(\word_t \mid \words_{<t}) \defeq \sum_{\word' \in \vocab} \similarityfunc(\word_t, \word')\, p(\word' \mid \words_{<t})
    & \!\!\!\!\!\!\!\!\!\!\mathcomment{expectation in \diversesurprisal} \\
    &\pz(\word_t \mid \meaningvar, \words_{< t}) \defeq \sum_{\word' \in \vocab} \similarityfunc(\word_t, \word')\, p(\word' \mid \meaningvar, \words_{< t})
    & \mathcomment{analogous to $\pz(\word_t \mid \words_{<t})$} \\
    &\pz(\meaningvar \mid \words_{<t}) \defeq p(\meaningvar \mid \words_{<t}) 
    & \mathcomment{does not depend on $\word_t$} \\
    &\pz(\meaningvar \mid \words_{< t} \circ \word_t) \defeq \frac{\pz(\word_t \mid \meaningvar, \words_{< t})\,\pz(\meaningvar \mid \words_{< t})}{\pz(\word_t \mid \words_{< t})} 
    & \mathcomment{Bayes-inspired definition}\label{eq:bayes_insp}
\end{align}
\end{subequations}
where \cref{eq:bayes_insp} is ``Bayes-inspired'' because $\pz$ is not necessarily a valid probability distribution (it does not necessarily sum to 1 across its support) and so the equivalence given by Bayes theorem is not guaranteed; rather it is an equivalence that we enforce in the definition of $\pz$.

We now present a theorem linking \diversesurprisal and processing cost, under the definitions above.
\begin{restatable}{theorem}{semanticsurprisaltheorycost} 
\label{theorem:semanticsurprisaltheorycost}
\textbf{Cost equals \diversesurprisal}.
Under standard assumptions\cref{footnote} about $\pz(\meaningvar, \words_{<t} \circ \word_t)$ and using the definitions in \cref{eq:pz_definitions}, we can show that:
\begin{align}
    \diversecostfunc(\word_t) = \diversesurp(\word_t)
\end{align}
\end{restatable}

\begin{proof}
First, we provide a helpful equivalence for $\pz(\meaningvar \mid \words_{< t} \circ \word_t)$:
\begin{subequations}
\begin{align}
    \pz(\meaningvar \mid \words_{< t} \circ \word') 
    &= \frac{\pz(\word' \mid \meaningvar, \words_{< t})\,\pz(\meaningvar \mid \words_{< t})}{\pz(\word' \mid \words_{< t})} \\
    &= \frac{\Big(\sum_{\word' \in \vocab} \similarityfunc(\word_t, \word')\, p(\word' \mid \meaningvar, \words_{< t})\Big)\,p(\meaningvar \mid \words_{< t})}{\sum_{\word' \in \vocab} \similarityfunc(\word_t, \word')\, p(\word' \mid \words_{< t})} 
    & \mathcomment{expand $\pz$} \\
    &= \frac{\similarityfunc(\word_t, \word_t)p(\word_t \mid \meaningvar, \words_{< t})\,p(\meaningvar \mid \words_{< t})}{\sum_{\word' \in \vocab} \similarityfunc(\word_t, \word')\, p(\word' \mid \words_{< t})}
    & \!\!\!\!\!\!\!\!\!\!\!\!\!\!\!\!\!\!\!\!\!\!\!\!\!\!
    \mathcomment{deterministic $p(\word' \mid \meaningvar, \words_{< t})$}\\
    &= \frac{p(\meaningvar \mid \words_{< t})}{\sum_{\word' \in \vocab} \similarityfunc(\word_t, \word')\, p(\word' \mid \words_{< t})} 
    & \!\!\!\!\!\!\!\!\!\!\!\!\!\!\!\!\!\!\!\!\!\!\!\!\!\!\!\!\!\!
    \!\!\!\!\!\!\!\!\!\!\!\!\!\!\!\!\!\!\!\!\!\!\!\!\!\!\!\!\!\! 
    \mathcomment{$p(\word_t \mid \meaningvar, \words_{< t})=1$}
\end{align}
\end{subequations}
Given these equalities, we can follow the same logic as in \cref{theorem:surprisaltheorycost} to show that processing cost in the presence  of similarity-aware distributions over words has an equivalence with \diversesurprisal:
\begin{subequations}
    \begin{align}
        \diversecostfunc(\word_t \mid \words_{<t}) 
        &= \diversekl(p(\meaningvar \mid \words_{< t} \circ \word_t) \mid\mid p(\meaningvar \mid \words_{<t})) \\
        &= \sum_{\meaningvar \in \meaningspace} p(\meaningvar \mid \words_{< t} \circ \word_t) \log \frac{\pz(\meaningvar \mid \words_{< t} \circ \word_t)}{\pz(\meaningvar \mid \words_{< t})} \\
        &= \sum_{\meaningvar \in \meaningspace} p(\meaningvar \mid \words_{< t} \circ \word_t) \log \frac{p(\meaningvar \mid \words_{< t})}{\Big(\sum_{\word' \in \vocab} \similarityfunc(\word_t, \word')\, p(\word' \mid \words_{< t})\Big)\, p(\meaningvar \mid \words_{< t})} \\
        &= \sum_{\meaningvar \in \meaningspace} p(\meaningvar \mid \words_{< t} \circ \word_t) \log \frac{1}{\sum_{\word' \in \vocab} \similarityfunc(\word_t, \word')\, p(\word' \mid \words_{< t})} \\
        &= \log \frac{1}{\sum_{\word' \in \vocab} \similarityfunc(\word_t, \word')\, p(\word' \mid \words_{< t})} \\
        &= - \log \pz(\word_t \mid \words_{< t}) \\
        %
        &= \diversesurp(\word_t)
    \end{align}
\end{subequations}
where $\diversekl$ is defined analogously to both standard $\kl$ and \diverse entropy: 
it implements the same expectation as standard $\kl$, but takes the $\log$ of distributions $\pz$ instead.
\end{proof}

\section{Related Work in Psycholinguistics}\label{app:related_work} 

\newcommand{\postag}{\textsc{pos}}
\newcommand{\postagspace}{\mathcal{C}}

In this section, we review some related work in more detail, and when possible connect it to \diversesurprisal. 
\citet{arehalli-etal-2022-syntactic} investigate a word's syntactic surprisal---i.e., the surprisal associated with the syntactic structure
implied by that word---as a predictor of reading comprehension behavior. 
They define this value as:
\begin{align}
    - \log \sum_{\word' \in \vocab} \underbrace{\sum_{\postag \in \postagspace} p(\postag \mid \words_{<t} \circ \word_t)\,p(\postag \mid \words_{<t} \circ \word')}_{\textit{potential choice of }\similarityfunc(\word_t, \word')}\,p(\word' \mid \words_{<t})
\end{align}
where, in their case, $\postag \in \postagspace$ represents a combinatory categorial grammar \citep[CCG][]{steedman1987combinatory} supertag.
We provide a more general measure of predictability here, as we can realize their notion of syntactic surprisal in our \diversesurprisal framework by using a similarity function that identifies equivalent syntactic classes.

\newcommand{\heuristicword}{\omega}
\newcommand{\heuristicwords}{\boldsymbol\omega}

More broadly speaking, many works have employed notions of semantic similarity in their models of language comprehension \cite[][\emph{inter alia}]{roland_2012,frank_willems_2017,giulianelli-etal-2023-information,giulianelli-etal-2024-generalized,giulianelli-etal-2024-incremental,li-2023}. 
For example, \citet{li-2023} offers a decomposition of a word's surprisal into two quantities, where they specifically consider the word as perceived by a comprehender $\heuristicword_t$:\looseness=-1%
\begin{equation}
\surp(\heuristicword_t) \defeq -\log p(\heuristicword_t \mid \heuristicwords_{<t})   =  \underbrace{\mathbb{E}\left[-\log p(\word_t \mid \words_{<t}) \right]}_{\textit{heuristic surprise}} + \underbrace{\mathbb{E}\left[\log \frac{p(\word_t \mid \words_{<t})}{p(\heuristicword_t \mid \heuristicwords_{<t})} \right]}_{\textit{discrepancy signal}}
\end{equation}
Here, $\word_t$ represents the ground truth word at time $t$, which they call a ``heuristic word''.   
Similarly to our semantic variant of \diversesurprisal, they use a notion of semantic distance to estimate the latter quantity. Our works differ in several ways though, most notably in that we do not propose a new model of language comprehension, but rather introduce an alternative definition of surprisal.

\section{Experimental Setup} \label{app:exp_setup}
Code for reproducing experimental results can be found at \url{https://github.com/cimeister/diverse-surprisal}. 
\subsection{Data}
We use four reading time datasets. 
Per-word reading time is measured according to one of two paradigms: self-paced and eye-tracked reading.
The self-paced reading corpora are the Natural Stories Corpus \citep{futrell-etal-2018-natural} and the Brown Corpus \citep{smith2013-log-reading-time}.
The eye-tracking corpora are the Provo Corpus \citep{provo} and the Dundee Corpus \citep{dundee}. 
We refer to the original works for further details on data collection. 

Before computing our different word-level predictors, text from all corpora was pre-processed using the Moses decoder\footnote{\url{http://www.statmt.org/moses/}} tokenizer and punctuation normalizer. Additional pre-processing was performed by the tokenizers for respective neural models. Capitalization was kept intact albeit we used the lowercase version of words when querying for unigram frequency estimates. We estimate unigram frequencies following \citet{nikkarinen-etal-2021-modeling} on the WikiText 103 dataset.

\subsection{Information-Theoretic Estimators}
We estimate surprisal and information value\footnote{We use the codebase of \citet{giulianelli-etal-2023-information} to compute information value. For variants of \diversesurprisal and information value that require estimators, we use 50 samples in all experiments.} using \gptsmall{} \cite{gpt2};\footnote{We use the open-source version available on the \texttt{transformers} library \citep{wolf-etal-2020-transformers}.} while not the most accurate language model in terms of perplexity, prior work has shown \gptsmall{} to have better psycholinguistic predictive power than its larger counterparts \cite{oh2023why,shain2024large}. 
Note that \gptsmall{} operates over subwords while reading time measurements are taken at the word level. 
We discuss our approaches for accounting for this characteristic for each estimator separately. 
\paragraph{Surprisal.}
We query our language model for next token probabilities; our surprisal estimate for a token is then simply the negative log of this value. 
To compute word-level estimates of surprisal, we sum these values across the tokens that constitute each word (as delineated by the reading time dataset). 
In general, surprisal decomposes additively across subunits of a word, theoretically grounding this approach.
However, as \citet{pimentel2024compute} point out, the way that subword units demarcate the beginning of a word complicates the computation of word-level surprisal estimates. They offer a simple fix for this issue, which we do not incorporate here since extending it to information value and \diversesurprisal is non-trivial. 
See also \citet{oh2024leading} for a similar discussion, and \citet{giulianelli-etal-2024-proper} for further discussion on the role of tokenization in computational psycholinguistics, as well as for a method to compute the surprisal of any character span from token-level language models.

\paragraph{Embedding-based information value and \diversesurprisal estimators.}
We likewise use \gptsmall{} for our word embeddings. 
Transformer-based language models can provide word embeddings for all of the (sub)words in their vocabulary. 
Thus, in this setting, we take $\vocab$ to be \gptsmall's subword vocabulary. 
We explore both contextual and non-contextual word embeddings in the computation of \cref{eq:sem_sim}. 
We use layer 0 for non-contextual and layer 12 (the last layer) for contextual embeddings; we leave the exploration of the use of other embedding functions, e.g., other language models, layers or aggregation across layers, to future work. 
In the case of contextual embeddings, obtaining embeddings for each $\word \in\vocab$ requires a separate query to the language model. 
Querying the model $|\vocab|\approx 50,000$ times for every context $\words_{<t}$ would be very computationally intensive, so we instead use a Monte Carlo estimator for these variants of information value and \diversesurprisal. 
Specifically, similarly to \citet{giulianelli-etal-2024-generalized}, we sample next tokens (with replacement) according to $\ptheta(\cdot\mid\words_{<t})$. 
We query our language model for only the embeddings for these tokens, and use them to make a Monte Carlo estimator of \cref{eq:mc,eq:div_surp}.\footnote{For non-contextual embeddings, we compute \cref{eq:mc,eq:div_surp} exactly, with \gptsmall's vocabulary as $\vocab$.} 
To create next-word information value and \diversesurprisal estimates from these subword-level estimates, we sum across these values for each of the tokens that constitute a word. 
Note that information value and \diversesurprisal do not decompose across subwords. 
We ran experiments where predictors were set to the value of the first subword that constituted a word and observed similar results; we omit these to reduce clutter. 

\paragraph{POS and edit distance estimators.}
These estimators cannot be computed at the subword level. Thus, we consider a full-word vocabulary. 
Because of the computational load that it would require to get language model estimates for a comprehensive vocabulary, we instead use a Monte Carlo estimator. 
Explicitly, we sample full-word continuations (with replacement) according to $\ptheta(\cdot\mid\words_{<t})$, sampling subwords until we reach either a white space marker or the end-of-sentence token. 
Note that this also allows us to avoid the task of explicitly defining a full-word vocabulary.  
We then use these continuations to build Monte Carlo estimators of \cref{eq:mc,eq:div_surp}. 

\section{Additional Experimental Results}\label{app:add_results}

\subsection{From \Diverse to Standard Surprisal}
In this experiment we equip $\similarityfunc$  with a temperature parameter $\alpha$, exponentiating our definition of word pair similarity in \cref{eq:sem_sim} by a given $\alpha$. 
As $\alpha\mathop{\rightarrow}\infty$, all $\similarityfunc(\word, \word')$ for $\word\neq \word'$ go to $0$; on the other hand $\similarityfunc(\word, \word)$ remains at $1$. 
Thus, \diversesurprisal converges to standard surprisal as $\alpha \rightarrow \infty$. 
We observe how the psycholinguistic predictive power of \diversesurprisal changes with $\alpha$ in \cref{fig:temp_llh}. 
While varying $\alpha$ does not appear to have a significant effect on the predictive power of \diversesurprisal, we see that, as expected, the $\deltallh$ of $\diversesurpalpha$ converges to that of surprisal as $\alpha\mathop{\rightarrow}\infty$. 

\begin{figure*}[h]
    \centering
    \includegraphics[width=\textwidth]{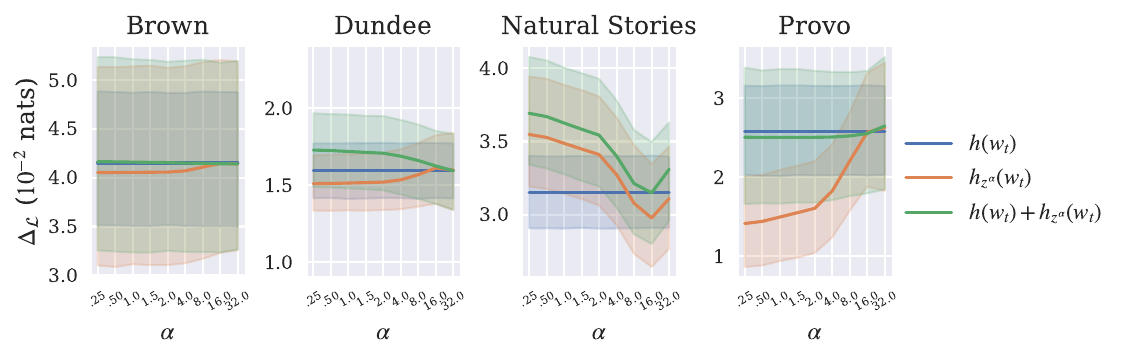}
    \caption{The change in reading time dataset log-likelihoods as a function of the temperature parameter used with the semantic-similarity function in \diversesurprisal computations. Each line corresponds to a different set of predictors added to the regressor. Shaded regions indicate 95\% confidence intervals, as computed using standard bootstrapping techniques on our per-fold $\deltallh$ values. }
    \label{fig:temp_llh}
\end{figure*}

\clearpage

\subsection{\texorpdfstring{$\deltallh$}{Delta LLH} of Different Predictors}
\begin{table*}[h]
\centering
\resizebox{\textwidth}{!}{
\begin{tabular}{lccccccccc}
\toprule
  &\textbf{Surprisal} & \multicolumn{4}{c}{\textbf{\Diversesurprisal}} & \multicolumn{4}{c}{\textbf{Information value}} \\
   \cmidrule(lr){3-6} \cmidrule(lr){7-10}
   && Non-contextual & Contextual & POS & Orthographic & Non-contextual & Contextual & POS & Orthographic
     \\
\midrule
Brown & \phantom{-}\textcolor{mygreen}{4.15}$^{***}$& \phantom{-}\textcolor{mygreen}{2.78}$^{***}$& {-0.02}\phantom{$^{***}$}& \phantom{-}{2.22}\phantom{$^{***}$}& \phantom{-}\textcolor{mygreen}{1.29}$^{*}$\phantom{$^{**}$}& \phantom{-}\textcolor{mygreen}{2.61}$^{***}$& {-0.02}\phantom{$^{***}$}& \phantom{-}\textcolor{mygreen}{0.33}$^{*}$\phantom{$^{**}$}& \phantom{-}\textcolor{mygreen}{0.71}$^{*}$\phantom{$^{**}$} \\
Dundee & \phantom{-}\textcolor{mygreen}{1.60}$^{***}$& \phantom{-}\textcolor{mygreen}{1.44}$^{***}$& \phantom{-}{0.01}\phantom{$^{***}$}& \phantom{-}{0.69}\phantom{$^{***}$}& \phantom{-}\textcolor{mygreen}{0.64}$^{*}$\phantom{$^{**}$}& \phantom{-}\textcolor{mygreen}{1.43}$^{***}$& \phantom{-}{0.01}\phantom{$^{***}$}& \phantom{-}\textcolor{mygreen}{0.26}$^{***}$& \phantom{-}\textcolor{mygreen}{0.48}$^{***}$ \\
Natural Stories & \phantom{-}\textcolor{mygreen}{3.15}$^{***}$& \phantom{-}\textcolor{mygreen}{3.18}$^{***}$& \phantom{-}\textcolor{mygreen}{0.31}$^{**}$\phantom{$^{*}$}& \phantom{-}\textcolor{mygreen}{1.04}$^{*}$\phantom{$^{**}$}& \phantom{-}\textcolor{mygreen}{0.79}$^{*}$\phantom{$^{**}$}& \phantom{-}\textcolor{mygreen}{3.18}$^{***}$& \phantom{-}\textcolor{mygreen}{0.30}$^{***}$& \phantom{-}\textcolor{mygreen}{0.46}$^{***}$& \phantom{-}\textcolor{mygreen}{0.88}$^{***}$ \\
Provo & \phantom{-}\textcolor{mygreen}{2.58}$^{***}$& \phantom{-}\textcolor{mygreen}{1.34}$^{***}$& \phantom{-}{0.05}\phantom{$^{***}$}& \phantom{-}{1.49}\phantom{$^{***}$}& \phantom{-}{0.82}\phantom{$^{***}$}& \phantom{-}\textcolor{mygreen}{1.14}$^{***}$& \phantom{-}{0.06}\phantom{$^{***}$}& \phantom{-}{0.57}\phantom{$^{***}$}& \phantom{-}{0.83}\phantom{$^{***}$} \\

\bottomrule
\end{tabular}}%
\vspace{-5pt}

\caption{$\deltallh$ (in $10^{-2}$ nats) on reading time data of regressors with different predictors over baseline regressors (i.e., regressors with only baseline predictors). Variable values for current and previous 3 words are provided as predictors. MC estimates of information value and \diversesurprisal use 50 samples per context.}
\label{tab:delta_llh_control_baseline}
\end{table*}

\begin{table}[!ht]
\centering
\resizebox{\textwidth}{!}{
\begin{tabular}{lcccccccc}
\toprule
& \multicolumn{4}{c}{\textbf{\Diversesurprisal}} & \multicolumn{4}{c}{\textbf{Information value}} \\
   \cmidrule(lr){2-5} \cmidrule(lr){6-9}
   & Non-contextual & Contextual & POS & Orthographic & Non-contextual & Contextual & POS & Orthographic
     \\
\midrule
Brown & \textcolor{myred}{-1.37}$^{***}$& \textcolor{myred}{-4.17}$^{***}$& {-1.94}\phantom{$^{***}$}& \textcolor{myred}{-2.86}$^{***}$& \textcolor{myred}{-1.54}$^{***}$& \textcolor{myred}{-4.17}$^{***}$& \textcolor{myred}{-3.82}$^{***}$& \textcolor{myred}{-3.44}$^{***}$ \\
Dundee & \textcolor{myred}{-0.16}$^{*}$\phantom{$^{**}$}& \textcolor{myred}{-1.59}$^{***}$& {-0.91}\phantom{$^{***}$}& \textcolor{myred}{-0.96}$^{***}$& \textcolor{myred}{-0.17}$^{*}$\phantom{$^{**}$}& \textcolor{myred}{-1.59}$^{***}$& \textcolor{myred}{-1.34}$^{***}$& \textcolor{myred}{-1.12}$^{***}$ \\
Natural Stories & \phantom{-}{0.03}\phantom{$^{***}$}& \textcolor{myred}{-2.84}$^{***}$& \textcolor{myred}{-2.11}$^{**}$\phantom{$^{*}$}& \textcolor{myred}{-2.37}$^{***}$& \phantom{-}{0.03}\phantom{$^{***}$}& \textcolor{myred}{-2.85}$^{***}$& \textcolor{myred}{-2.69}$^{***}$& \textcolor{myred}{-2.27}$^{***}$ \\
Provo & \textcolor{myred}{-1.24}$^{**}$\phantom{$^{*}$}& \textcolor{myred}{-2.54}$^{***}$& {-1.09}\phantom{$^{***}$}& \textcolor{myred}{-1.76}$^{*}$\phantom{$^{**}$}& \textcolor{myred}{-1.44}$^{***}$& \textcolor{myred}{-2.52}$^{***}$& \textcolor{myred}{-2.02}$^{**}$\phantom{$^{*}$}& \textcolor{myred}{-1.75}$^{**}$\phantom{$^{*}$} \\

\bottomrule
\end{tabular}}%
\vspace{-5pt}
\caption{$\deltallh$ (in $10^{-2}$ nats) on reading time data of regressors with our different predictors of interest in comparison to regressors \emph{with} surprisal (i.e., replacing all surprisal terms for current and previous words with corresponding \diversesurprisal/information value terms). MC estimates of information value and \diversesurprisal use 50 samples per context.}
\label{tab:delta_llh_replace_surprisal}
\end{table}

\begin{table*}[h]
\centering
\resizebox{\textwidth}{!}{
\begin{tabular}{lcccccccc}
\toprule
   & \multicolumn{4}{c}{\textbf{\Diversesurprisal}} & \multicolumn{4}{c}{\textbf{Information value}} \\
   \cmidrule(lr){2-5} \cmidrule(lr){6-9}
   & Non-contextual & Contextual & POS & Orthographic & Non-contextual & Contextual & POS & Orthographic
     \\
\midrule
Brown & \phantom{-}{0.01}\phantom{$^{***}$}& {-0.01}\phantom{$^{***}$}& \phantom{-}{1.00}\phantom{$^{***}$}& {-0.03}\phantom{$^{***}$}& \phantom{-}{0.02}\phantom{$^{***}$}& {-0.01}\phantom{$^{***}$}& \phantom{-}{0.00}\phantom{$^{***}$}& \phantom{-}{0.01}\phantom{$^{***}$} \\
Dundee & \phantom{-}\textcolor{mygreen}{0.12}$^{***}$& {-0.00}\phantom{$^{***}$}& \phantom{-}{0.38}\phantom{$^{***}$}& \phantom{-}\textcolor{mygreen}{0.22}$^{**}$\phantom{$^{*}$}& \phantom{-}\textcolor{mygreen}{0.14}$^{***}$& {-0.00}\phantom{$^{***}$}& {-0.00}\phantom{$^{***}$}& \phantom{-}{0.00}\phantom{$^{***}$} \\
Natural Stories & \phantom{-}\textcolor{mygreen}{0.47}$^{***}$& \phantom{-}\textcolor{mygreen}{0.27}$^{**}$\phantom{$^{*}$}& \phantom{-}{0.03}\phantom{$^{***}$}& \phantom{-}{0.02}\phantom{$^{***}$}& \phantom{-}\textcolor{mygreen}{0.57}$^{***}$& \phantom{-}\textcolor{mygreen}{0.29}$^{**}$\phantom{$^{*}$}& {-0.01}\phantom{$^{***}$}& \phantom{-}{0.00}\phantom{$^{***}$} \\
Provo & \textcolor{myred}{-0.08}$^{***}$& \phantom{-}{0.01}\phantom{$^{***}$}& {-0.03}\phantom{$^{***}$}& \phantom{-}{0.07}\phantom{$^{***}$}& {-0.07}\phantom{$^{***}$}& {-0.01}\phantom{$^{***}$}& {-0.03}\phantom{$^{***}$}& \textcolor{myred}{-0.10}$^{**}$\phantom{$^{*}$} \\
\bottomrule
\end{tabular}}%
\vspace{-5pt}
\caption{$\deltallh$ (in $10^{-2}$ nats) over baseline \emph{with} surprisal when adding a \diversesurprisal or information value term for (only) the current word $w_t$. Monte Carlo (MC) estimation with 50 samples.}
\label{tab:delta_llh_current}
\end{table*}

\end{document}